\documentclass[letterpaper, 10 pt, conference]{ieeeconf}

\IEEEoverridecommandlockouts                              
\usepackage{acronym}
\let\theoremstyle\undefined
\usepackage{amsmath}
\usepackage{amsthm}
\usepackage{amssymb}
\usepackage{mathtools}
\usepackage{amsfonts}
\usepackage{algorithm}
\usepackage{algorithmicx}
\usepackage[noend]{algpseudocode}
\usepackage{csquotes}
\usepackage{paralist}
\usepackage{tikz} 
\usepackage{url}
\usetikzlibrary{arrows,automata}
\usetikzlibrary{shapes}
\usetikzlibrary{positioning}
\usepackage[final]{changes}
\usepackage{hyperref}
\hypersetup{
    colorlinks=true,
    linkcolor=blue,
    filecolor=magenta,      
    urlcolor=cyan,
    }

\newif\ifuseboldmathops
\newif\ifuseittextabbrevs
\useboldmathopstrue   

\ifuseittextabbrevs
	\newcommand{\ie}{{\it i.e.}}
\else
	\newcommand{\ie}{i.e.}
\fi

\ifuseboldmathops

\else

\fi

\ifuseboldmathops

\else

\fi

\ifuseboldmathops

\else

\fi

\ifuseboldmathops

\else

\fi

\newcommand{\rank}{\mathop{\mathrm{rank}}}

\newcommand{\supp}{\mbox{Supp}}

\newcommand{\calF}{\mathcal{F}}
\newcommand{\calW}{\mathcal{W}}

\newcommand{\occ}{{\mathsf{Occ}}}

\acrodef{mdp}[MDP]{Markov Decision Process}
\acrodef{pomdp}[POMDP]{Partially Observable Markov Decision Process}
\acrodef{ltl}[LTL]{Linear Temporal Logic}
\acrodef{dfa}[DFA]{Deterministic Finite Automaton}
\acrodef{dfpa}[DFPA]{Deterministic Finite-State Preference Automaton}
\acrodef{spi}[SPI]{Safe and Positive Improving}
\acrodef{sasi}[SASI]{Safe and Almost-Sure Improving}

\theoremstyle{definition}
\newtheorem{definition}{Definition}
\newtheorem{example}{Example}
\newtheorem{problem}{Problem}
\newtheorem{lemma}{Lemma}

\newtheorem{proposition}{Proposition}
\newtheorem{corollary}{Corollary}

\newtheorem{theorem}{Theorem}

\newcommand{\asw}{\mathsf{ASWin}}

\newcommand{\poswin}{\mathsf{PWin}}

\acrodef{smdp}[Semi-MDP]{Semi-Markov decision process}
\acrodef{mcts}[MCTS]{Monte Carlo tree search}
\acrodef{uct}[UCT]{Upper Confidence Bound 1 applied to trees}
\acrodef{scltl}[scLTL]{syntactically co-safe LTL}
\acrodef{ssp}[SSP]{Stochastic Shortest Path}
\acrodef{p2sg}[SG(2)]{Two-player Stochastic Game}
\acrodef{mc}[MC]{Markov chain}
\acrodef{prefltl}[TPL]{ Temporal Preference Logic}
\acrodef{tld}[TLwD]{Temporal Logic with Distributions}
\acrodef{mtl}[Metric TL]{Metric Temporal Logic}
\acrodef{sta}[STA]{Stochastic Timed Automaton}

\newcommand{\dist}{\mathcal{D}}
\newcommand{\plays}{\mathsf{Plays}}

\newcommand{\prefplays}{\mathsf{PrefPlays}}

\newcommand{\calM}{\mathcal{M}}

\newcommand{\reach}[1]{\mathsf{Reach}(#1)}

\acrodef{gpf}[GPF]{generalized preference formula}

\acrodef{cp}[CP]{ceteris paribus}
\acrodef{milp}[MILP]{Mixed-Integer Linear Programming}
\acrodef{dfa}[DFA]{Deterministic Finite Automaton}

\newcommand{\strictpref}{\triangleright}
\newcommand{\weakpref}{\trianglerighteq}

\newcommand{\outcomes}{\mathsf{Outcomes}}

\newcommand{\powerset}{\wp}

\newcommand{\prefix}{\mathsf{PrefPlays}}

\newcommand{\cone}{\mathsf{Cone}}

\newcommand{\pref}{{\nu}}
\title{Opportunistic Qualitative Planning in Stochastic Systems with Incomplete Preferences over Reachability Objectives}

\author{Abhishek N. Kulkarni$^\ast$, and Jie Fu
\thanks{A. Kulkarni ($^\ast$ corresponding author) and J. Fu are with the Dept. of Electrical and Computer Engineering, University of Florida, Gainesville, Fl 32611 USA.
{\tt\small \{a.kulkarni2,fujie\}@ufl.edu}}
\thanks{This material is based upon work supported by the Air Force Office of Scientific Research under award number FA9550-21-1-0085.}
}

\begin{document}

\maketitle

\begin{abstract}
Preferences play a key role in determining what goals/constraints to satisfy when not all constraints can be satisfied simultaneously.	In this paper, we study how to synthesize preference satisfying plans in stochastic systems, modeled as a \ac{mdp}, given a (possibly incomplete) combinative preference model over temporally extended goals. We start by introducing new semantics to interpret preferences over infinite plays of the stochastic system. Then, we introduce a new notion of \emph{improvement} to enable comparison between two prefixes of an infinite play. Based on this, we define two solution concepts called \ac{spi} and \ac{sasi} that enforce improvements with a positive probability and with probability one, respectively. We construct a model called an improvement \ac{mdp}, in which the synthesis of \ac{spi} and \ac{sasi} strategies that guarantee at least one improvement reduces to computing positive and almost-sure winning strategies in an \ac{mdp}. We present an algorithm to synthesize the \ac{spi} and \ac{sasi} strategies that induce multiple sequential improvements. We demonstrate the proposed approach using a robot motion planning problem. 
\end{abstract}

\section{Introduction}
	\label{sec:intro}
	With the rise of artificial intelligence, robotics and autonomous systems are being designed to make complex decisions by reasoning about multiple goals at the same time. Preference-based planning (PBP) allows the systems to decide which goals to satisfy when not all of them can be achieved \cite{hastie2010rational}. Even though PBP has been studied since the early 1950s, most works on preference-based temporal planning (c.f. \cite{baier2008planning}) fall into at least one of the following categories: (a) those which assume that all outcomes are pairwise comparable---that is, the preference relation is \emph{complete} \cite{son2006planning,bienvenu2011specifying}, (b) those which study exclusionary preferences---that is, the set of outcomes is mutually exclusive (see \cite{stanford2022preferences} and the references within), (c) those which are interpreted over finite traces \cite{rahmani2022probabilistic}. In this work, we study the PBP problem for the class of systems in which the preference model is \emph{incomplete}, \emph{combinative} (as opposed to exclusionary) and is interpreted over \emph{infinite plays} of the stochastic system.

The motivation to study incomplete, combinative preferences comes from two well-known facts that the assumption of completeness is strong and, in many cases, unrealistic \cite{aumann1962utility}, and that combinative preferences are more expressive than exclusionary preferences \cite{hansson2001structure}. In many control applications, preferences may need to admit incompleteness because of (a) \emph{Inescapability}: An agent has to make decisions under time limits but with incomplete information about preferences because, for example, it lost communication with the server; and (b) \emph{Incommensurability}: Some situations, for instance, comparing the quality of an apple to that of banana, are fundamentally incomparable since they lack a standard basis to compare. Similarly, the common preferences in robotics such as ``visiting A is preferred to visiting B'' are better interpreted under a combinative model because a path visiting A may pass through B, which means that the outcomes (sets of plays of the stochastic model satisfying a certain property) are not mutually exclusive. 



Preference-based planning problems over temporal goals have been well-studied for deterministic planning given both complete and incomplete preferences (see \cite{baier2008planning} for a survey). For preference specified over temporal goals, several works \cite{tumova2013least, wongpiromsarn2021, rahmani2020what} proposed minimum violation planning methods that decide which low-priority constraints should be violated in a deterministic system. Mehdipour \emph{et al.}~\cite{mehdipourSpecifyingUserPreferences2021} associate weights with Boolean and temporal operators in signal temporal logic to specify the importance of satisfying the sub-formula and priority in the timing of satisfaction. This reduces the PBP problem to that of maximizing the weighted satisfaction in deterministic dynamical systems. However, the solutions to the PBP problem for deterministic systems cannot be applied to stochastic systems. This is because, in stochastic systems, even a deterministic strategy yields a distribution over outcomes. Hence, to determine a better strategy, we need a comparison of distributions---a task a deterministic planner cannot do.

Several works have studied the PBP problem for stochastic systems. Lahijanian and Kwiatkowska~\cite{Lahijanian2016} considered the problem of revising a given specification to improve the probability of satisfaction of the specification. They formulated the problem as a multi-objective \ac{mdp} problem that trades off minimizing the cost of revision and maximizing the probability of satisfying the revised formula. Li \emph{et al.}~\cite{li2020probabilistic} solve a preference-based probabilistic planning problem by reducing it to a multi-objective model checking problem. However, all these works assume the preference relation to be \emph{complete}. To the best of our knowledge, \cite{fu2021probabilistic} is the only work that studies the problem of probabilistic planning with incomplete preferences. The authors introduce the notion of the value of preference satisfaction for planning within a pre-defined finite time duration and developed a mixed-integer linear program to maximize the satisfaction value for a subset of preference relations. 


The aforementioned work studied preference-based quantitative planning. In comparison, this work focuses on qualitative planning in \ac{mdp}s with preferences over a set of outcomes, representable by reachability objectives.  We first introduce new semantics to interpret an incomplete, combinative preference model over infinite plays of a stochastic system. We observe that uncertainties in the planning environment combined with infinite plays might give rise to opportunities to \emph{improve} the outcomes achieved by the agent. Thus, analogous to the idea of an \emph{improving flip} \cite{santhanamRepresentingReasoningQualitative2016}, we define the notion of \emph{improvement} that compares two prefixes of an infinite play to determine which one is more preferred, based on their different prospects regarding the set of possible, achievable objectives. Based on whether a strategy exists to enforce an improvement with a positive probability or with probability one, we introduce two solution concepts called \emph{safe and positively improving} and \emph{safe and almost-surely improving} which ensure an improvement can be made with positive probability and with probability one, respectively. The synthesis of  \ac{spi} and \ac{sasi} strategies is through a construction called \emph{improvement \ac{mdp}} and a reduction to that of computing positive and almost-sure winning strategies for some reachability objectives of the improvement \ac{mdp}. In the case of almost-surely improvement, we also provide an algorithm to determine the maximum number of improvements achievable given any given state. The correctness of the proposed algorithms is demonstrated through a robot motion planning example.

\section{Preliminaries}
	\label{sec:prelim}  
	\textbf{Notation.} Given a finite set $X$, the powerset of $X$ is denoted as $\powerset(X)$. The set of all finite (resp., infinite) ordered sequences of elements from $X$ is denoted by $X^\ast$ (resp., $X^\omega$). The set of all finite   ordered sequences of length $>0$ is denoted by $X^+$.
We write $\dist(X)$ to denote the set of probability distributions over $X$. The support of a distribution $D \in \dist(X)$ is denoted by $\supp(D) = \{x \in X \mid D(x) > 0\}$.

In this paper, we consider a class of decision-making problems in stochastic systems modeled as a \ac{mdp} without the reward function. We then introduce a preference model over the set of infinite plays in the \ac{mdp}.

\begin{definition}[MDP]
	An \ac{mdp} is a tuple $M = \langle S, A, T, \iota \rangle,$ where $S$ and $A$ are finite state and action sets, $\iota \in S$ is an initial state, and $T: S \times A \rightarrow \dist(S)$ is the transition probability function such that $T(s, a, s')$ is the probability of reaching the state $s' \in S$ when action $a \in A$ is chosen at the state $s \in S$. 
\end{definition}

A \emph{play} in an \ac{mdp} $M$ is an infinite sequence of states $\rho = s_0 s_1 \ldots \in S^\omega$ such that, for every integer $i \geq 0$, there exists an action $a \in A$ such that $T(s_i, a, s_{i+1})>0$. We denote the set of all plays starting from $s \in S$ in the \ac{mdp} by $\plays(M, s)$ and the set of all plays in $M$ is denoted by $\plays(M) = \bigcup_{s \in S}\plays(M, s)$. 
The set of states occurring in a play is given by $\occ(\rho) = \{s \in S \mid \exists i \ge 0, s_i = s\}$. A prefix of a play $\rho$ is a finite sub-sequence of states $\nu = s_0 s_1 \ldots s_{k}$, $k \geq 0$, whose the length is $|\nu| = k+1$. The set of all prefixes of a play $\rho$ is denoted by $\mathsf{Pref}(\rho)$. The set of all prefixes in $M$ is denoted by $\prefix(M) = \cup_{\rho \in \plays(M)} \mathsf{Pref}(\rho)$. Given a prefix $\nu = s_0 s_1 \ldots s_k \in \prefix(M)$, the sequence of states $s_{k+1} s_{k+2} \ldots \in S^\omega$ is called a suffix of $\nu$ if the play $\nu \rho' = s_0 s_1 \ldots s_k s_{k+1} s_{k+2} \ldots$ is an element of $\plays(M)$.

In this \ac{mdp}, we consider reachability objectives for the agent. Given a set $F \subseteq S$, a reachability objective is characterized by the set $\reach{F} = \{\rho \in \plays(M) \mid \occ(\rho) \cap F \neq \emptyset\}$, which contains all the plays in $M$ starting at the state $s \in S$ that visit $F$. Any play $\rho \in \plays(M)$ that satisfies a reachability objective $\reach{F}$ has a \emph{good prefix} $\nu \in \mathsf{Pref}(\rho)$ such that the last state of $\nu$ is in $F$ \cite{kupferman2001model}.

A finite-memory (resp., memoryless), non-deterministic strategy in the \ac{mdp} is a mapping $\pi: S^+ \rightarrow \powerset(A)$ (resp., $\pi: S \rightarrow \powerset(A)$) from a prefix to a subset of actions that can be taken from that prefix. The set of all finite-memory, nondeterministic strategies is denoted $\Pi$.
Given a prefix $\nu = s_0 \ldots s_k \in \prefix(M)$, a suffix $\rho = s_{k+1} s_{k+2} \ldots \in S^\omega$ is \emph{consistent} with $\pi$, if for all $i \ge 0$, there exists an action $a \in \pi(s_0 \ldots s_{k} \ldots s_{k+i})$ such that $T(s_i, a, s_{i+1})>0$. 
Given an \ac{mdp} $M$, a prefix $\nu \in \prefix(M)$ and a strategy $\pi$, the \emph{cone} is defined as the set of consistent suffixes of $\nu$ w.r.t. $\pi$, that is  \[\cone(M, \nu, \pi) = \{\rho \in S^\omega \mid \nu \rho \text{ is consistent with } \pi \}.\] 


Given a prefix $\pref \in \prefix(M)$ and a reachability objective $\reach{F}$, a (finite-memory/memoryless) strategy $\pi^{\poswin(F)}$ is said to be positive winning   if $\cone(M, \pref, \pi)\cap \reach{F}\ne \emptyset$. Similarly, a (finite-memory/memoryless) strategy $\pi^{\asw(F)}$ is said to be almost-sure winning   if $  \cone(M, \nu, \pi) \subseteq \reach{F}$. 


The set of \emph{states} in the \ac{mdp} $M$, starting from which the agent has an almost-sure (resp. positive) winning strategy to satisfy a reachability objective $F \in \mathbb{F}$ is called the almost-sure (resp., positive) winning region and is denoted by $\asw(F)$ (resp., $\poswin(F)$). The almost-sure and positive winning strategies in the product \ac{mdp} are known to be memoryless. The almost-sure winning region and strategies can be synthesized in polynomial time and linear time, respectively \cite{baier2008principles}. 

\section{Preference Model}
	\label{sec:pref-model}  

\begin{definition}
	A preference model is a tuple $\langle U, \succeq \rangle$, where $U$ is a countable set of outcomes and $\succeq$ is a reflexive and transitive binary relation on $U$.
\end{definition}

Given $u_1, u_2 \in U$, we write $u_1 \succeq u_2$ if $u_1$ is \emph{weakly preferred to} (\ie,  is at least as good as) $u_2$; and $u_1\sim u_2$ if $u_1\succeq u_2$ and $u_2 \succeq u_1$, that is,  $u_1$ and $u_2$ are \emph{indifferent}. We write $u_1 \succ u_2$ to mean that $u_1$ is \emph{strictly preferred} to $u_2$, \ie, $u_1\succeq u_2$ and $u_2 \not\succeq u_1$.  We write $u_1 \nparallel u_2$ if $u_1$ and $u_2$ are \emph{incomparable}. Since Def.~\ref{def:preference-model} allows outcomes to be incomparable, it models \emph{incomplete} preferences \cite{bouyssou2009concepts}.

We consider   planning objectives specified as  preferences over reachability objectives. 

\begin{definition}
\label{def:preference-model}
	A preference model over reachability objectives in an \ac{mdp} $M$ is a tuple $\langle \mathbb{F}, \weakpref \rangle$, where $\mathbb{F} = \{\reach{F_1}, \reach{F_2}, \ldots, \reach{F_n}\}$ is a set of reachability objectives such that $F_1, \ldots, F_n$ are subsets of $S$.
\end{definition}

Intuitively, a preference $\reach{F_1} \weakpref \reach{F_2}$ means that any play in $\reach{F_1}$ is weakly preferred to any play in $\reach{F_2}$. The strict preference ($\strictpref$), indifference and incomparability are understood similarly.

The model $\langle \mathbb{F}, \weakpref \rangle$ is a combinative preference model, as opposed to exclusionary one. This is because we do not assert the exclusivity condition $\reach{F_1} \cap \reach{F_2} = \emptyset$. This allows us to represent a preference such as ``Visiting A and B is preferred to visiting A,'' where the less preferred outcome must be satisfied first in order to satisfy the more preferred outcome. In literature, it is common to study exclusionary preference models (see \cite{baier2008planning,bienvenu2011specifying} and the references within) because of their simplicity \cite{stanford2022preferences}. However, we focus on planning with combinative preferences since they are more expressive than the exclusionary ones \cite{hansson2001structure}. That is, every exclusionary preference model can be transformed into a combinative one, but the opposite is not true. 

%

When a combinative preference model is interpreted over infinite plays, the agent needs a way to compare the sets of reachability objectives satisfied by two plays. For instance, in the example from previous paragraph, to compare a play that visits A and B with a play that visits only A, the agent must compare the sets $\{\reach{F_{B}}, \reach{F_A}\}$ with $\{\reach{F_A}\}$. Since visiting A and B is more preferred than visiting A, first play is preferred over the second. However, if the preference was ``visiting A is preferred over visiting B'', then the two plays would be indifferent since both visit the more preferred objective. In this case, the less preferred objective of visiting B does not influence the comparison of the sets. To formalize this notion, we define notion of most-preferred outcomes. 



Given a non-empty subset $\mathbb{X} \subseteq \mathbb{F}$, let $\mathsf{MP}(\mathbb{X}) \triangleq \{R \in \mathbb{X} \mid \nexists R' \in \mathbb{X}: R' \strictpref R\}$ denote the set of most-preferred outcomes in $\mathbb{X}$.
 
\begin{definition}
\label{def:mp-outcomes}
	Given a preference model $\langle \mathbb{F}, \weakpref \rangle$ and a play $\rho \in \plays(M)$, the set of most-preferred outcomes satisfied by $\rho$ is given by $\mathsf{MP}(\rho) \triangleq \mathsf{MP}(\{\reach{F} \in \mathbb{F} \mid \exists \nu \in \mathsf{Pref}(\rho): \nu \text{ is a good prefix for } \reach{F}\})$. 
\end{definition}

By definition, there is no outcome included in $\mathsf{MP}(\rho)$ that is preferred to any other outcome in $\mathsf{MP}(\rho)$. Thus, we have the following result.

\begin{lemma}
	For any play $\rho \in \plays(M)$, every pair of outcomes in $\mathsf{MP}(\rho)$ is incomparable to each other. 
\end{lemma}

Now, we formally define the interpretation of $\langle \mathbb{F}, \weakpref \rangle$ in terms of the preference relation it induces on $\plays(M)$.

\begin{definition}
	Let $\langle \plays(M), \succeq \rangle$ be the preference model induced by $\langle \mathbb{F}, \weakpref \rangle$. Then, for any $\rho_1, \rho_2 \in \plays(M)$, we have
	
	\begin{itemize}
		\item $\rho_1 \succ \rho_2$ if and only if there exist a pair of outcomes $R \in \mathsf{MP}(\rho_1)$ and $R' \in \mathsf{MP}(\rho_2)$ such that $R \strictpref R'$, and there does not exist a pair of outcomes $R \in \mathsf{MP}(\rho_1)$ and $R' \in \mathsf{MP}(\rho_2)$ such that $R' \strictpref R$.
		
		\item $\rho_1 \sim \rho_2$ if and only if $\mathsf{MP}(\rho_1) = \mathsf{MP}(\rho_2)$.

		\item $\rho_1 \nparallel \rho_2$, otherwise.
	\end{itemize}
\end{definition}

\section{Solution Concept}
	\label{sec:problem}
	In preference-based planning, the agent is to choose its next action given a finite prefix $\pref \in \prefix(M)$ in order to satisfy the given preference relation on a set of outcomes. A na\"ive approach to this problem is to follow the strategy to satisfy a most-preferred outcome from the set of almost-surely achievable outcomes given $\nu$. However, this is not sufficient as illustrated by the following example.

\begin{example}
	Consider the toy \ac{mdp} shown in the Fig.~\ref{fig:illustration}. To clarify, the exact probabilities are omitted. The transitions are understood as follows: Given action $a$ at state $s_0$, it is possible to reach both $s_5$ and $s_1$ with positive probabilities.
	Given the three sets $F_1 = \{s_1, s_5\}, F_2 = \{s_2, s_4\}$ and $F_3 = \{s_3\}$, let $\langle \mathbb{F}, \weakpref \rangle$ be the preference model such that $\mathbb{F} = \{\reach{F_1}, \reach{F_2}, \reach{F_3}\}$ and $\reach{F_2} \strictpref \reach{F_1}$ and $\reach{F_3} \strictpref \reach{F_1}$. Consider the state $s_0$ at which the agent is to choose its next action. From $s_0$, the agent can visit $F_1$ almost surely by choosing the action $a$. It, however, does not have an almost sure winning strategy to visit either $F_2$ or $F_3$, individually. But, by choosing action $b$ at $s_0$, the agent will almost surely visit either $F_2$ or $F_3$ and, thereby, achieve an outcome strictly better than $F_1$. 
	
	\begin{figure}
		\centering
        \begin{tikzpicture}[->,>=stealth',shorten >=1pt,auto,node distance=2.5cm, scale = 0.5,transform shape,align=center]
		\node[state] (5) {\LARGE $s_5$};
		\node[state] (0) [right of=5] {\LARGE $s_0$};
		\node[state] (4) [right of=0] {\LARGE $s_4$};
		\node[state] (1) [below=1.5cm of 5] {\LARGE $s_1$};
		\node[state] (2) [right of=1] {\LARGE $s_2$};
            \node[state] (3) [right of=2] {\LARGE $s_3$};
		
		\path 
		(0) edge node {\LARGE $a$} (1)
		(0) edge node {\LARGE $b$} (2)
		(0) edge node {\LARGE $b, c$} (3)
		(0) edge node {\LARGE $c$} (4)
		(0) edge node {\LARGE $a$} (5)
		(5) edge node {\LARGE $a$} (1)
		
		;
	\end{tikzpicture}
	
		\caption{Toy example to illustrate the limitation of almost-sure winning solution concept for preference-based planning. The states with no outgoing transitions are sink states.}
		\label{fig:illustration}
	\end{figure}
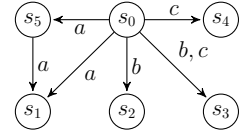
\end{example}

The example highlights that almost-sure winning solution concept is not suitable for preference-based planning because it reasons about exactly one outcome at a time. As a result, the agent cannot reason about opportunities to achieve a better outcome that may become available to due to stochasticity in the environment.

In the sequel, we introduce two new solution concepts for probabilistic planning under incomplete preferences interpreted over infinite plays. Our solution concepts are based upon the notion of an \emph{improvement} that generalizes the idea of \emph{improving flip} \cite{santhanamRepresentingReasoningQualitative2016} which is defined for propositional preferences. An improving flip compares two outcomes representable as propositional logic formulas to determine which is more preferred. Analogously, an improvement compares two prefixes of a play to determine which one can yield a more preferred outcome with probability one.

Given a prefix $\nu$, let $\outcomes(\nu) = \{\reach{F} \in \mathbb{F} \mid \exists \pi \in \Pi, \forall \rho \in \cone(M, \nu, \pi) : \rho \in \reach{F}\}$ be the set of outcomes, each of which can be achieved almost-surely under some strategy. Note that different outcomes may require different policies to achieve them.

\begin{definition}
	\label{def:improvement} 
	Given a play $\rho \in \plays(M)$ and two of its prefixes $\nu, \nu' \in \mathsf{Pref}(\rho)$ such that $|\nu'| > |\nu|$, $\nu'$ is said to be an \emph{improvement} of $\nu$ if there exists a pair of outcomes $R \in \mathsf{MP}(\outcomes(\nu))$ and $R' \in \mathsf{MP}(\outcomes(\nu'))$ such that $R' \strictpref R$. And, $\nu'$ is said to be a \emph{weakening} of $\nu$ if there exists a pair of outcomes $R \in \mathsf{MP}(\outcomes(\nu))$ and $R' \in \mathsf{MP}(\outcomes(\nu'))$ such that $R \strictpref R'$.
\end{definition}

Given a prefix $s_0 s_1 \ldots s_k \in \prefix(M)$, the transition from $s_{k-1}$ to $s_k$ is said to be an \emph{improving transition} if the prefix $s_0 s_1 \ldots s_{k-1} s_k$ is an improvement over $s_0 s_1 \ldots s_{k-1}$. A play that contains an improving transition is called an \emph{improving play}. It is noted that a prefix $\nu'$ can simultaneously be an improvement and a weakening of a prefix $\nu$.

Next, we define the two solution concepts that, while avoiding any weakening, induce improvements either with positive probability or with probability one.



\begin{definition}[SPI/SASI Strategy]
\label{def:spi-strategy}
	Given a prefix $\nu = s_0 s_1 \ldots s_k \in \prefix(M)$, a strategy $\pi: S^+ \rightarrow 2^A$ is said to be \emph{safe and positively (resp., safe and almost-surely) improving} for $\nu$ if the following conditions hold:
	\begin{enumerate}
		\item (Safety) For all $\rho \in \cone(M, \nu, \pi)$, the play $\nu \rho$ satisfies that $s_0 s_1 \ldots s_j$ is not a weakening of $s_0 s_1 \ldots s_k$ for any integer $j > k$.

		\item (Improvement) There exists (resp., for any)  $\rho \in \cone(M, \nu, \pi)$, the play $\nu \rho$ satisfies the condition that there exists an integer $j > k$ such that $s_0 s_1 \ldots s_j$ is an improvement over $s_0 s_1 \ldots s_k$. 
	\end{enumerate}
\end{definition}



We now state our problem statement. 

\begin{problem}
\label{problem}
	Given an \ac{mdp} $M$ and a preference model $\langle \mathbb{F}, \weakpref\rangle$, design an algorithm to synthesize an \ac{spi} and a \ac{sasi} strategy. 
\end{problem}

\section{Opportunistic Qualitative Planning with Incomplete Preferences}  
	\label{sec:planning}


Our approach to synthesize \ac{spi} and \ac{sasi} strategies distinguishes between \emph{opportunistic} states, \ie, the states from which an improvement could be made, and \emph{non-opportunistic} states. We now introduce a new model called \emph{an improvement \ac{mdp}} to synthesize the \ac{spi} and \ac{sasi} strategies. 

To facilitate the definition,  we slightly abuse the notation and let $\mathsf{MP}(s) \triangleq \mathsf{MP}(\{\reach{F} \in \mathbb{F} \mid s \in \asw(F)\})$ be the set of outcomes almost-surely achievable from state $s$ in $M$.

\begin{definition}[Improvement \ac{mdp}]
\label{def:improvement-mdp}
	Given an \ac{mdp} $M = \langle S, A, T, s_0 \rangle$ and a preference model $\langle \mathbb{F}, \weakpref \rangle$, an \emph{improvement \ac{mdp}} is the tuple,
	\[
		\calM = \langle V, A, \Delta, v_0, \calF \rangle,
	\] 
	where $V = S \times \{0, 1\}$ is the set of states, $A$ is the same set of actions as $M$, $v_0 = (s_0, 0)$ is the initial state, and $\calF = \{(s, 1) \mid s \in S\}$ is a set of final states. The transition function $\Delta: V \times A \rightarrow \dist(V)$ is defined as follows: For any states $v = (s, m), v' = (s', m') \in V$ and for any action $a \in A$, $\Delta(v, a, v') > 0$ holds if and only if the following conditions hold: 
	\begin{enumerate}
		\item $T(s, a, s') > 0$. 
		
		\item (Safety) For all pairs of outcomes $R \in \mathsf{MP}(s)$ and $R' \in \mathsf{MP}(s')$, we have $R \not \strictpref R'$. 
		
		\item (Improvement) If there exists a pair $R \in \mathsf{MP}(s)$ and $R' \in \mathsf{MP}(s')$ such that $R' \strictpref R$, then $m' = 1$ else $m' = 0$. 
	\end{enumerate} 
\end{definition}

Every play $\rho = s_0 s_1 \ldots \in \plays(M)$ induces a play $\varrho = v_0 v_1 \ldots$ in $\calM$ such that for all $i = 0, 1, \ldots$, $v_i = (s_i, m_i)$ where $m_i \in \{0, 1\}$ represents a memory element that captures whether the transition from $s_{i-1}$ to $s_i$ is improving. The following proposition highlights important features of the improvement \ac{mdp}. Before that we note the following fact to prove Proposition~\ref{prop:imdp-properties}.

\begin{lemma}
	\label{lma:memoryless}
	For every prefix $\nu = s_0 s_1 \ldots s_k \in \prefplays(M)$, it holds that $\outcomes(\nu) = \outcomes(s_k)$ and thus $\mathsf{MP}(\outcomes(\nu)) = \mathsf{MP}(\outcomes(s_k))$.
\end{lemma}
The proof follows from the fact that memoryless strategies are sufficient to ensure the satisfaction of reachability objectives in \ac{mdp}s \cite{de2001quantitative}. In other words, if an outcome is  almost-surely  achievable given a prefix $\nu= s_0s_1\ldots s_k$, then it is almost-surely achievable given $s_k$.

For convenience, we will write $\mathsf{MP}(s) = \mathsf{MP}(\outcomes(s))$ to denote the set of most preferred outcomes satsfiable/achievable with some strategy from a state $s \in S$.

\begin{proposition}
\label{prop:imdp-properties}
	For any play $\varrho = v_0 v_1 \ldots \in \plays(\calM)$ such that $v_i = (s_i, m_i)$ for all $i = 0, 1 \ldots$, the following statements hold. 
	\begin{enumerate}
		\item (Safety). For every prefix $v_0 v_1 \ldots v_j \in \prefix(\varrho)$, $s_0 s_1 \ldots s_j$ is not a weakening of $s_0 s_1 \ldots s_i$ for any $0 \leq i < j$.
		

		\item (Improvement). For every integer $k > 0$ such that $s_k \in \calF$, the prefix $s_0 s_1 \ldots s_k$ is an improvement of $s_0 s_1 \ldots s_{k-1}$. 		
	\end{enumerate}
\end{proposition}
\begin{proof}[Proof (Sketch)]
	For statement (1) to hold, it must be the case that $R \not\strictpref R'$  holds for all pairs of outcomes $R \in \mathsf{MP}(s_{i})$ and $R' \in \mathsf{MP}(s_{j})$. This is true because of Lma.~\ref{lma:memoryless} and the fact that every transition from $v_i$ to $v_{i+1}$, $j < i \leq k$, that violates the condition is disabled by Def.~\ref{def:improvement-mdp}.  
	
	To see why statement (2) holds, consider an integer $k > 0$ such that $v_k \in \calF$. Then, by construction, there exists a pair $R \in \mathsf{MP}(s_{k-1})$ and $R' \in \mathsf{MP}(s_{k})$ such that $R' \strictpref R$. 
\end{proof}

In words, the improvement \ac{mdp} guarantees by construction that no play in $\plays(\calM)$ violates the safety condition of Def.~\ref{def:spi-strategy}. Moreover, it helps identify the opportunistic states as the ones that have an outgoing transition into $\calF$.

\begin{corollary}	
	A play $\varrho \in \plays(\calM)$ is improving if and only if $\occ(\varrho) \cap \calF \neq \emptyset$.
\end{corollary}

As a result, the problem of determining whether an improvement is possible from a state $v \in V$ reduces to checking whether a state in $\calF$ can be reached from $v$ with a positive probability (in case of \ac{spi} strategy) or with probability one (in case of \ac{sasi} strategy).

\begin{theorem}
	\label{thm:spi-sasi-strategies}
	The following statements hold:
	\begin{enumerate}
		\item The positive winning strategy $\pi^{\poswin(\calF)}$ in $\calM$ to visit $\calF$ is an \ac{spi} strategy.
		\item The almost-sure winning strategy $\pi^{\asw(\calF)}$ in $\calM$ to visit $\calF$ is a \ac{sasi} strategy.
	\end{enumerate}
\end{theorem}

The proof follows from the fact that there exists a (resp., every) play $\rho \in \cone(\calM, v_0, \pi)$ induced by any positive (resp., almost-sure) winning strategy $\pi$ visits $\calF$ with positive probability (resp., probability one) \cite{chatterjee2012survey}. Therefore, Thm.~\ref{thm:spi-sasi-strategies} establishes that by following $\pi^{\poswin(\calF)}$ (resp., $\pi^{\asw(\calF)}$), the agent is ensured to make an improvement with a positive probability (resp., with probability one).

\begin{algorithm}[t]
	\caption{Level set for constructing \ac{sasi} strategy}
 	\label{alg:competitive-strategy}
	
	\begin{algorithmic}[1]
		\item[\textbf{Inputs:}] Improvement \ac{mdp}, $\calM = \langle \rangle $.
		\item[\textbf{Outputs:}] Level set, $\calW$.
		
		\State $i \gets 0$ 
		\State $R_i \gets \calF$
		\While{$R_i$ is not empty}
		\State $W_{i+1} \gets \asw(R_i)$
		\State $R_{i+1} \gets \{(s, 1) \in \calF \mid (s, 0) \in W_{i+1}\}$
		\If{$i = 0$}
		\State Add $V \setminus W_{i+1}$ to level $0$ in $\calW$. 
		\EndIf
		\State Add $W_{i+1}$ to level $i+1$ in $\calW$. 
		\State $i \gets i + 1$
		\EndWhile
		\State \Return $\calW$
	\end{algorithmic}
\end{algorithm}

The \ac{spi} and \ac{sasi} strategies from Thm.~\ref{thm:spi-sasi-strategies} guarantee that at least one improvement will occur with positive probability or with probability one. Next, we present Alg.~\ref{alg:competitive-strategy}, using which we can determine the maximum number of improvements that can \emph{almost-surely} be made from a given state in $\calM$. The algorithm to determine the maximum number of improvements possible from a given state in $\calM$ with \emph{a positive probability} and its properties are similar to Alg.~\ref{alg:competitive-strategy}.

First, note the following properties of the improvement \ac{mdp} which follow from the construction of \ac{mdp}.

\begin{proposition}
\label{prop:symmetric-outgoing-acts}
	Consider two states $(s, 0), (s, 1) \in V$, it holds that for any action $a \in A$, we have $\supp(\Delta((s, 0), a)) = \supp(\Delta((s, 1), a))$.
\end{proposition}
The proof is straightforward because given $(s,0), (s,1)$, for any action $a\in A$, if a transition from $s$ to $s'$ given $a$ is improving, then $\Delta((s,0),a,(s',1))>0 $ and $\Delta((s,1),a,(s',1))>0$. Else, $\Delta((s,0),a,(s',0))>0 $ and $\Delta((s,1),a,(s',0))>0 $.

\begin{corollary}
\label{cor:symmetry-in-imdp}
	The final states $\calF$ can be visited \emph{again} from a state $(s, 1) \in V$ with a positive probability (resp., with probability one) if and only if $\calF$ can be visited from $(s, 0)$ with a positive probability (resp., with probability one). 
\end{corollary}
\begin{proof}
	Let $\pi$ be a positive winning strategy to visit $\calF$ from $(s, 0)$. Let $Y = \supp(\Delta((s, 0), a))$ for some $a \in \pi((s, 0))$. By the property of a positive winning strategy, a state in $\calF$ is reached with positive probability by following $\pi$ from any state in $Y$. By Proposition~\ref{prop:symmetric-outgoing-acts}, $Y = \supp(\Delta((s, 1), a))$. Therefore, by choosing $a$ at $(s, 1)$ and then following $\pi$, a state in $\calF$ is visited with positive probability from $(s, 1)$. The proof for almost-sure winning is similar. 
\end{proof}

Intuitively, Alg.~\ref{alg:competitive-strategy} constructs a set $\calW$ of level sets such that from any state that appears at $k$-th level in $\mathcal{W}$, at least $k$ visits to $\calF$ are guaranteed and, thereby, at least $k$ improvements can be made.

For this purpose, it iteratively computes the almost-sure winning region to visit the states in $R_i \subseteq \calF$, from which $\calF$ can be visited at least $i$ times. We denote by $W_i$ the $i$-th level set. 
The level-$0$ of $\calW$ contains the states $V \setminus \asw(\calF)$ from which $\calF$ cannot be visited again with probability one. That is, $0$-visits to $\calF$ are guaranteed from any state in level-$0$ of $\calW$. Every state in level-$1$ of $\calW$ is almost-surely winning to visit $\calF$. Hence, at least one visit to $\calF$ is guaranteed. Now, consider the subset $R_1 = \{(s, 1) \in \calF \mid (s, 0) \in W_{1}\}$ of final states $\calF$. By Corollary~\ref{cor:symmetry-in-imdp}, because $(s,0)\in W_1 = \asw(\calF)$, there exists a strategy from every state in $R_1$ to visit $\calF$ with probability one. Therefore, from any state $(s, 0) \in W_2 = \asw(R_1)$ at least two improvements are guaranteed---first, when visiting $(s', 1) \in R_1$ and, second, when visiting $R_0 = \calF$ by following the almost-sure winning strategy at $(s', 1)$. Repeating a similar argument, it follows that at least $k$-visits are guaranteed almost-surely from states at $k$-th level in $\calW$. 

The largest integer $k \geq 0$ such that the state $(s, 0) \in V$ appears at $k$-th level of $\calW$ is called the rank of the states $(s, 0)$ and $(s, 1)$, denoted as $\rank(s, 0) = \rank(s, 1) = k$.

\begin{proposition}
\label{prop:sasi-rank}
    From any state $v = (s, m) \in V$, $m \in \{0, 1\}$, there exists a strategy to visit $\calF$ at least $\rank(v)$-many times. 
\end{proposition}
\begin{proof}
We prove this by constructing the strategy that achieves $\rank(v)$ improvements: First, if $\rank(v)=k$, then by construction it is in $\asw(R_{k-1})$. Following the almost-sure winning strategy a state in $R_{k-1}$ can be reached with probability one and thus the first improvement is made. Upon reaching a state, say $(s',1)$, in $R_{k-1}$, one identify $(s',0)\in W_{k-1}$. Because $W_{k-1}= \asw(R_{k-2})$, an almost-sure winning strategy exists to reach $R_{k-2}$ and hence the second improvement. Repeating the similar steps, eventually $R_0$ will be reached after the $k$-th improvement.
\end{proof}

\begin{corollary}
    From any state $v = (s, m) \in V$ at most $\rank(v)$-many visits to $\calF$ are almost-surely guaranteed. 
\end{corollary}
\begin{proof}[Proof (Sketch)]
    By contradiction. Let $\rank(v) = k$. Suppose that $k + 1$ visits are possible from $v$. Then, following the argument in the proof of Proposition~\ref{prop:sasi-rank}, on making $k$-th visit to $\calF$, the resulting state must still be in $W_1$ so that $k+1$-th visit to $\calF$ could be made. By definition, this means that the state from which $k$-th visit is made is also contained in $W_2$. Using a similar argument repeatedly, it must be the case that $v \in W_{k+1}$, which means that $\rank(v) = k+1$---a contradiction.
\end{proof}

\textbf{Complexity.} Alg.~\ref{alg:competitive-strategy} runs in polynomial time with respect to the size of $\calM$ since the while loop can run no more than $|V|$ times and the complexity of $\asw$ is quadratic in the size of $\calM$ \cite{baier2008principles}. 

\section{Example: Robot Motion Planning}  
	\label{sec:illustration}
	We illustrate our approach using a motion planning problem for a robot in a $5 \times 5$ gridworld as shown in Figure~\ref{fig:gridworld}. The gridworld environment consists of seven regions: $\{A: (0, 0), B: (2, 0), C: (4, 0), D: (2, 4), E: (4, 4), F: (1, 2)\}$ from which the robot must pick up an item. There is a charging station at cell $(4, 2)$. Each cell denoted using the convention \texttt{(row, col)}. The robot can choose among four actions \texttt{N, S, E, W} to deterministically move north, east, south and west by one cell. The actions \texttt{E, W} are disabled in the cells $(4, 2)$ and $(2, 2)$. The cells $(1, 1), (3, 1), (1, 3), (3, 3)$ are slippery, that is, whenever the robot moves into any of these cells, say $(1, 1)$, it may non-deterministically end up in either the same cell $(1, 1)$, or the cell north to it $(2, 1)$, or south to it $(0, 1)$. In any cell, if applying an action results in a cell that is outside the gridworld or contains an obstacle, the robot returns to the same cell. The robot has limited battery of $5$ units, which it may recharge by visiting the charging station. The robot spends $1$ unit to execute each action.


At the beginning, only the items at $A, B$ and $C$ are available for pickup. That is, if the robot visits the charging station or regions $D, E, F$, then neither its battery will be recharged nor will it be able to pickup items $D, E, F$. When the robot picks up an item at $A$ or $B$, the charging station and the items at $D, E$ become available. When the robot picks up an item at $C$, the charging station and the items at $E, F$ become available. The following preference about picking up the items is given to the robot: $D \strictpref A, E \strictpref A, D \strictpref B, E \strictpref B, E \strictpref C, F \strictpref C$. By default, picking up any item is preferred to not picking up any item.

\begin{figure}
	\centering 
	\includegraphics[scale=0.35]{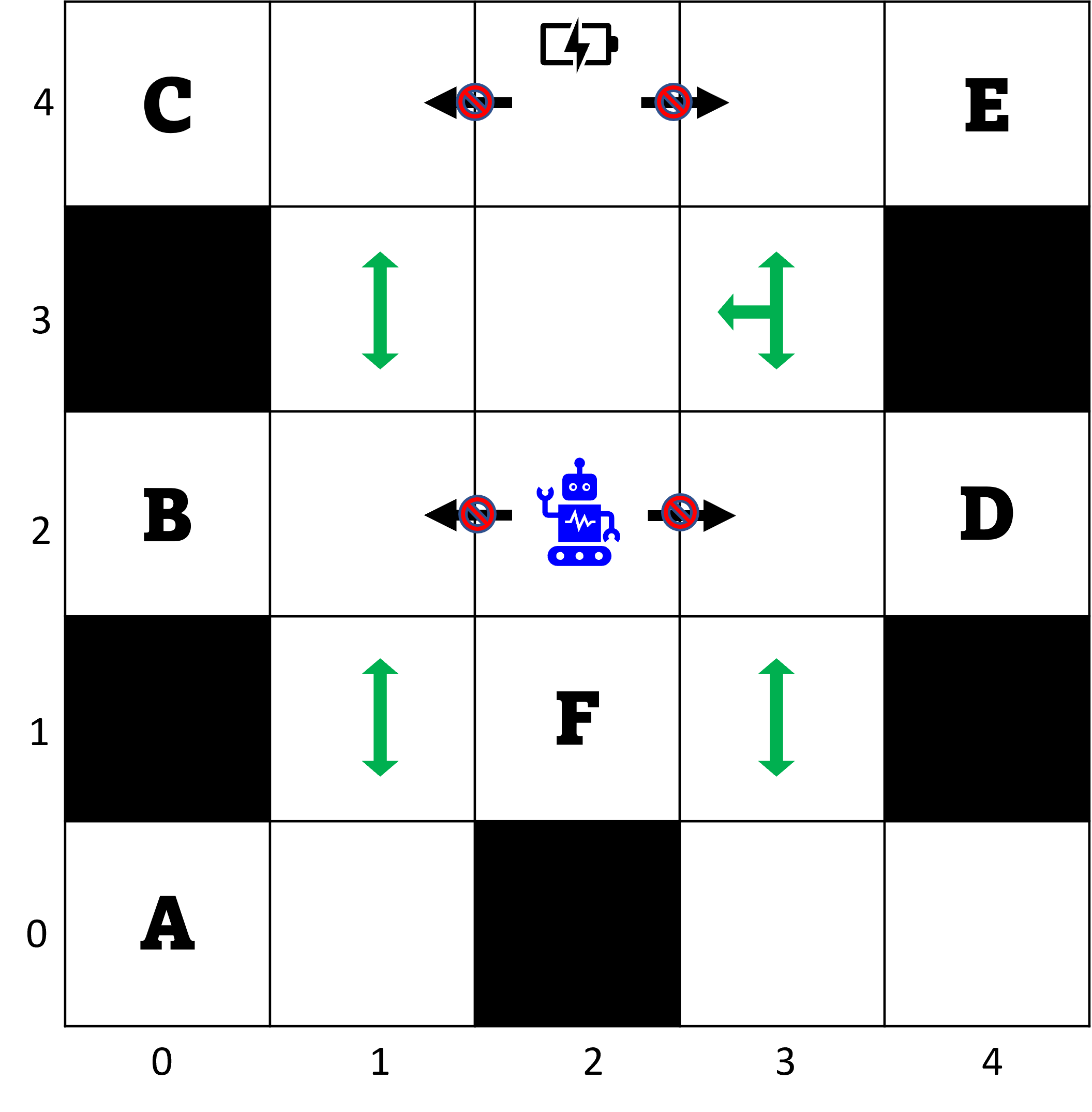}
	\caption{A gridworld example in which the black arrows with no-entry symbol denotes the disabled actions from that state and green arrows show the random outcomes on entering the cell.}
	\label{fig:gridworld}
\end{figure}

		
		


Note that the preference model given to the robot is incomplete as well as combinative. It is incomplete because picking up items $A, B, C$ are mutually incomparable outcomes. Similarly, picking up items $D, E, F$ are mutually incomparable. It is combinative because, for instance, any play in which robot picks up an item from $D$ or $E$ is considered preferred to a play in which robot only picks an item from $A$ or $B$, even though to pick an item from $D$ or $E$ an item from $A$ or $B$ must be picked first.


We implemented the example in Python 3.9 on a Windows 10 machine with a core i$7$, 2.80GHz CPU and a 32GB memory. We discuss few noteworthy observations next. The \ac{mdp} for this case has $3600$ states and $18496$ transitions, whereas the improvement \ac{mdp} has $7200$ states and $35524$ transitions.

Consider the initial state $s_0 = (2, 2, 8, (1, 1, 1, 0, 0, 0, 0))$ in which the robot is at cell $(2, 2)$ with $8$ units of battery. The fourth component of the state denotes which items are available for pickup with the last element of the tuple reserved for availability of the charging station. In this state, the robot does not have an almost-sure winning strategy to visit any of $A, B$ or $C$. This is because to visit, say $A$, the robot must visit the slippery cell $(1, 1)$. But whenever $(1, 1)$ is visited, the robot may reach $(2, 1)$ with a positive probability. Hence, $\mathsf{MP}(\outcomes(s_0)) = \emptyset$.

When under the \ac{sasi} concept, the rank of the state $(s_0, 1)$ is $2$ indicating that two improvements are almost-surely guaranteed. This is understood by observing the \ac{sasi} strategy which chooses action \texttt{N} at $(s_0, 0)$ to reach $s_1 = (3, 2, 7, (1, 1, 1, 0, 0, 0, 0))$. At $(s_1, 0)$ the strategy selects \texttt{W} and visits either $B$ or $C$ with probability one. Since a pickup from $B$ and $C$ are incomparable, both actions \texttt{N} and \texttt{S} are deemed valid under \ac{sasi} strategy at $(3, 1, 6, (1, 1, 1, 0, 0, 0, 0))$. On visiting either $B$ or $C$, the \ac{sasi} strategy follows the almost-sure winning strategy to visit either $D$ or $E$ to make a second improvement. Since visiting the cell $(3, 3)$ may result in returning back to the cell $(3, 2)$ with a positive probability, the robot can recharge itself until a successful visit to $E$ or $D$ is made.

The \ac{sasi} strategy at $(s_0, 0)$ does not select \texttt{S} because a second improvement cannot be guaranteed with probability one after visiting $A$ since the robot may remain at the cell $(0, 1)$ until its battery runs out. However, we observe that the \ac{spi} strategy at $(s_0, 0)$ allows selection of both actions \texttt{N, S} at $(s_0, 0)$ since in both cases two improvements are possible with positive probability.

We conclude with Table~\ref{tab:summary}, which shows the number of states from which the robot has an \ac{spi} and \ac{sasi} strategies to make at least $1$ or $2$ improvements, since the maximum number of improvements possible under given preference model is $2$. It is noted that the states from which a \ac{sasi} strategy exists are a subset of states from which a \ac{spi} strategy exists.

\begin{table}[t]
	\centering
		\begin{tabular}{||c || c | c ||} 
			
			\hline
			& SASI & SPI \\ [0.5ex] 
			\hline\hline
			Rank-$1$ & 768 & 926 \\ 
			\hline
			Rank-$2$ & 98 & 167 \\
			\hline
		\end{tabular}
	\caption{Number of states from which the robot has \ac{spi} and \ac{sasi} strategies to make at least $1$ or at least $2$ improvements.}
	\label{tab:summary}
\end{table}

\section{Conclusion}

In this paper, we introduced two solution concepts, namely \ac{spi} and \ac{sasi} to solve a preference-based planning problem given a combinative, incomplete preference model over infinite plays of a stochastic system. In the improvement \ac{mdp}, we showed that the synthesis of \ac{spi} and \ac{sasi} strategies reduces to that of computing positive and almost-sure winning strategies. Finally, we designed an algorithm using which we can synthesize a strategy that induces the maximum number of improvements under the \ac{sasi} concept. Building on this work, there are a number of future directions: 1) it is possible to  consider a preference over temporal objectives that encompass more general properties such as safety, recurrence, and liveness; 2) it remains open as how to connect qualitative reasoning with quantitative planning with such preference specifications. 

\bibliographystyle{IEEEtran}
\bibliography{refs}

\end{document}